\documentclass[11pt,english,bibliography=totoc]{extarticle}
\usepackage[T1]{fontenc}
\usepackage[latin9]{inputenc}
\usepackage{geometry}
\usepackage{color}
\usepackage{babel}
\usepackage{mathrsfs}
\usepackage{mathtools}
\usepackage{enumitem}
\usepackage{amsmath}
\usepackage{amsthm}
\usepackage{amssymb}
\usepackage{graphicx}
\usepackage[authoryear]{natbib}
\usepackage[unicode=true,pdfusetitle,
 bookmarks=true,bookmarksnumbered=false,bookmarksopen=false,
 breaklinks=false,pdfborder={0 0 0},pdfborderstyle={},backref=false,colorlinks=true]
 {hyperref}

\makeatletter
\numberwithin{figure}{section}
\theoremstyle{definition}
\newtheorem{defn}{\protect\definitionname}
\theoremstyle{plain}
\newtheorem{prop}{\protect\propositionname}
\theoremstyle{plain}
\newtheorem{lem}{\protect\lemmaname}
\theoremstyle{plain}
\newtheorem{thm}{\protect\theoremname}

\usepackage{dsfont}
\usepackage{relsize}
\usepackage{subdepth}
\usepackage{xcolor}
\allowdisplaybreaks

\usepackage{url}

\DeclareFontFamily{OT1}{pzc}{}
\DeclareFontShape{OT1}{pzc}{m}{it}{<-> s * [1.200] pzcmi7t}{}
\DeclareMathAlphabet{\mathpzc}{OT1}{pzc}{m}{it}

\usepackage[noadjust]{cite}
\usepackage{filecontents}

\usepackage{stfloats}

\usepackage{babel}

\renewcommand\footnotemark{}

\newtheorem{assumption}{Assumption}

\usepackage{bigints}

\usepackage{scalerel}

\usepackage[framemethod=tikz]{mdframed}

\newmdenv[
  hidealllines=true,
  backgroundcolor=blue!10,
  innerleftmargin=8pt,
  innerrightmargin=8pt,
  innertopmargin=0pt,
  innerbottommargin=6pt,
  leftmargin=-0pt,
  rightmargin=-0pt
]{shadedbox}

\usepackage[titles]{tocloft}
\definecolor{airforceblue}{rgb}{0.36, 0.54, 0.66}
\definecolor{ballblue}{rgb}{0.13, 0.67, 0.8}
\hypersetup{citecolor=ballblue}

\definecolor{alizarin}{rgb}{0.82, 0.1, 0.26}
\definecolor{asparagus}{rgb}{0.53, 0.66, 0.42}
\definecolor{applegreen}{rgb}{0.55, 0.71, 0.0}
\definecolor{armygreen}{rgb}{0.29, 0.33, 0.13}
\definecolor{amber(sae/ece)}{rgb}{1.0, 0.49, 0.0}
\definecolor{coquelicot}{rgb}{1.0, 0.22, 0.0}
\definecolor{ao(english)}{rgb}{0.0, 0.5, 0.0}

\newcommand{\CaseStretch}{1.2}

\renewcommand*\env@cases[1][\CaseStretch]{%
  \let\@ifnextchar\new@ifnextchar
  \left\lbrace
  \def\arraystretch{#1}%
  \array{@{}l@{\quad}l@{}}%
}

\usepackage{anyfontsize}
\usepackage{cleveref}

\makeatother

\providecommand{\definitionname}{Definition}
\providecommand{\lemmaname}{Lemma}
\providecommand{\propositionname}{Proposition}
\providecommand{\theoremname}{Theorem}

\begin{document}
\title{Noisy Linear Convergence of Stochastic Gradient Descent for\\
CV@R Statistical Learning under Polyak-\L ojasiewicz Conditions}
\author{Dionysios S. Kalogerias\\
Department of Electrical and Computer Engineering\\
Michigan State University}

\maketitle
\vspace{-18pt}

\begin{abstract}
Conditional Value-at-Risk \textbf{(}\textbf{\textit{$\mathrm{CV@R}$}})
is one of the most popular measures of risk, which has been recently
considered as a performance criterion in supervised statistical learning,
as it is related to desirable operational features in modern applications,
such as safety, fairness, distributional robustness, and prediction
error stability. However, due to its variational definition, \textbf{\textit{$\mathrm{CV@R}$
}}is commonly believed to result in difficult optimization problems,
even for smooth and strongly convex loss functions. We disprove this
statement by establishing noisy (i.e., fixed-accuracy) linear convergence
of stochastic gradient descent for sequential\textbf{\textit{ $\mathrm{CV@R}$}}
learning, for a large class of not necessarily strongly-convex (or
even convex) loss functions satisfying a set-restricted Polyak-\L ojasiewicz
inequality. This class contains all smooth and strongly convex losses,
confirming that classical problems, such as linear least squares regression,
can be solved efficiently under the \textbf{\textit{$\mathrm{CV@R}$
}}criterion, just as their risk-neutral versions. Our results are
illustrated numerically on such a risk-aware ridge regression task,
also verifying their validity in practice.
\end{abstract}
\textbf{\textit{$\quad$}}\textbf{Keywords.} Statistical Learning,
Risk-Aware Learning, Conditional Value-at-Risk, Stochastic Gradient
Descent, Stochastic Approximation, Polyak-\L ojasiewicz Inequality.

\section{Introduction}

Risk-awareness is becoming an increasingly important issue in modern
statistical learning theory and practice, especially due to the need
to meet strict reliability requirements in high-stakes, critical applications
\citep{Bennis2018,Ma2018,Kim2019,Cardoso2019,Koppel2019,Chaccour2020,Li2020}.
In such settings, risk-aware learning formulations are particularly
appealing, since they can \textit{explicitly balance} the performance
of optimal predictors between average-case and ``difficult'' to
learn, infrequent, or worst-case examples, inducing a form of \textit{statistical
robustness} in the learning outcome \citep{Takeda2009,W.Huang2017,Vitt2018,Cardoso2019,Zhou2020a,Soma2020,Gurbuzbalaban2020}.
The foundational idea of risk-aware statistical learning is to replace
the standard, expected loss learning objective by more general loss
functionals, called \textit{risk measures} \citep{ShapiroLectures_2ND},
whose purpose is to effectively quantify the statistical variability
of the random loss function considered, in addition to average performance.
Popular examples of risk measures include mean-variance functionals
\citep{Markowitz1952,ShapiroLectures_2ND}, mean-semideviations \citep{Kalogerias2018b},
and Conditional Value-at-Risk ($\mathrm{CV@R}$) \citep{Rockafellar1997}.

$\mathrm{CV@R}$, in particular, plays a significant role in supervised
statistical learning, as it is naturally connected not only to prediction
error stability (see Section 7), but also to distributional robustness
\citep{ShapiroLectures_2ND,Curi2019}, fairness \citep{Williamson2019},
as well as the formulation of classical learning problems, such as
the celebrated ($\nu$-)SVM \citep{Vapnik2000,Scholkopf2000,Takeda2008,Gotoh2016}.
Relevant generalization bounds were recently reported in \citep{Mhammedi2020}
and \citep{Lee2020}, establishing asymptotic consistency for $\mathrm{CV@R}$
learning, as well.

But except for operational effectiveness and generalization performance,
\textit{computational methods} for actually obtaining optimal solutions
to $\mathrm{CV@R}$ learning problems are of paramount importance,
especially for practical considerations. The design of such methods
is facilitated by the variational definition of $\mathrm{CV@R}$ (\citep{Rockafellar1997},
also see Section 2), allowing the reduction of any $\mathrm{CV@R}$
learning problem to a standard stochastic optimization problem with
a special loss function. This approach was followed in \citep{Soma2020},
where several averaged Stochastic Gradient Descent (SGD)-type algorithms
were analyzed under a batch setting (i.e., given a dataset available
\textit{a priori}). Almost concurrently, and under the same setting,
\citep{Curi2019} proposed an adaptive sampling algorithm for $\mathrm{CV@R}$
learning, by exploiting the distributionally robust representation
of $\mathrm{CV@R}$ \citep{ShapiroLectures_2ND}. In both works, convergence
rates reported are\textit{ at best} of the order of $1/{\textstyle \sqrt{T}}$,
where $T$ denotes the total runtime of the respective algorithm (iterations).

Such rates might seem to be nearly all we can get: Due to its construction,
\textbf{\textit{$\mathrm{CV@R}$ }}is commonly conjectured to result
in potentially difficult or badly behaved stochastic problems, mainly
because standard properties which enable fast convergence of gradient
methods, such as strong convexity, are \textit{not} preserved when
transitioning from (\textit{data-driven}) risk-neutral to \textbf{\textit{$\mathrm{CV@R}$}}
learning, \textit{even for} smooth and strongly convex losses.\textit{
}In this work, we disprove this argument by showing that SGD attains
\textit{noisy (i.e., fixed-tunable-accuracy) linear global convergence
}for sequential \textbf{\textit{$\mathrm{CV@R}$}} learning (i.e.,
provided a \textit{datastream}), for a large class of not necessarily
strongly-convex (or even convex) loss functions satisfying a \textit{set-restricted
Polyak-\L ojasiewicz inequality} \citep{Polyak1963,Karimi2016}. As
a byproduct of this result, we also obtain noisy linear convergence
of SGD for smooth and strongly convex losses, since those belong to
the aforementioned class. Essentially, our results confirm that at
least from an optimization perspective, \textbf{\textit{$\mathrm{CV@R}$}}
learning is almost as easy as risk-neutral learning. This implies
that \textbf{\textit{$\mathrm{CV@R}$}} learning can have widespread
use in applications, since risk-aware versions of ubiquitous problems,
such as linear least squares estimation, can be solved as efficiently
as their risk-neutral counterparts, and with provable \textit{and}
equivalent rate guarantees. Numerical simulations on such a basic
ridge regression task confirm the validity of our results in a practical
setting.

\section{\label{sec:CVAR Learning}$\boldsymbol{\mathrm{CV@R}}$ Statistical
Learning}

Let ${\cal P}_{{\cal D}}$ be an \textit{unknown} probability measure
over an \textit{example space} ${\cal D}\triangleq{\cal \mathbb{R}}^{d}\times\mathbb{R}$,
and consider a known parametric family of functions ${\cal F}\triangleq\{\phi:\mathbb{R}^{m}\rightarrow\mathbb{R}|\phi(\cdot)\equiv f(\cdot,\boldsymbol{\theta}),\boldsymbol{\theta}\in\mathbb{R}^{m}\}$,
called a \textit{hypothesis class}. We are interested in the problem
of discovering or \textit{learning} a function $f(\cdot,\boldsymbol{\theta}^{o})\in{\cal F}$
that \textit{best approximates} $y$ when presented with the input
$\boldsymbol{x}$, where the pair $(\boldsymbol{x},y)$ follows the
example distribution ${\cal P}_{{\cal D}}$. The instantaneous quality
of every admissible predictor $f(\cdot,\boldsymbol{\theta})$ is expressed
by a loss function $\ell:\mathbb{R}\times\mathbb{R}\rightarrow\mathbb{R}$
taking, for each example $(\boldsymbol{x},y)$, the quantities $f(\boldsymbol{x},\boldsymbol{\theta})$
and $y$ and mapping them to an integrable random variable, $\ell(f(\boldsymbol{x},\boldsymbol{\theta}),y)$.
Due to randomness on the example space, it is generally not possible
to minimize losses for all possible examples simultaneously. Instead,
it is standard to consider minimizing an expected loss functional
of the form
\begin{equation}
\hspace{-1bp}\inf_{\boldsymbol{\theta}\in\mathbb{R}^{m}}\hspace{-1bp}\bigg[\mathbb{E}_{{\cal P}_{{\cal D}}}\{\ell(f(\boldsymbol{x},\boldsymbol{\theta}),y)\}\hspace{-1bp}\equiv\hspace{-1bp}\int_{{\cal D}}\ell(f(\boldsymbol{x},\boldsymbol{\theta}),y)\mathrm{d}{\cal P}_{{\cal D}}(\boldsymbol{x},y)\hspace{-0.5bp}\bigg],\hspace{-1bp}\label{eq:RiskNeutral}
\end{equation}
which is at the heart of modern machine learning theory and practice
and beyond, such as signal processing, statistics, and control.

Despite its wide popularity, though, a fundamental issue with the
gold standard expected loss learning formulation is its very nature:
It is \textit{risk-neutral}, i.e., it minimizes losses \textit{only}
on average. Because of this, it lacks robustness and essentially ignores
\textit{relatively infrequent but statistically significant} example
instances, treating them as inconsequential. This is important from
a practical point of view, since such ``difficult'' or ``extreme''
examples will incur high and/or undesirable instantaneous losses,
\textit{even if} the optimal prediction error has minimal expected
value \citep{Takeda2009,ShapiroLectures_2ND,Kalogerias2018b,Koppel2019,Curi2019,Soma2020,Gurbuzbalaban2020}.

As briefly explained in Section 1, the need for a systematic treatment
of the shortcomings of the risk-neutral approach motivates and sets
the premise of \textit{risk-aware statistical learning}, in which
expectation is replaced by more general loss functionals, called risk
measures \citep{ShapiroLectures_2ND}. Their purpose is to induce
risk-averse characteristics into the learning outcome by explicitly
controlling the statistical variability of the random loss $\ell(f(\boldsymbol{x},\cdot),y)$,
or, equivalently, its tail behavior. By far one of the most popular
risk measures in theory and practice is \textbf{\textit{$\mathrm{CV@R}$}},
which for an integrable random loss $Z$ is defined as \citep{Rockafellar1997}
\begin{equation}
\mathrm{CV@R}^{\alpha}(Z)\triangleq\inf_{t\in\mathbb{R}}\Big\{ t+\dfrac{1}{\alpha}\mathbb{E}\{(Z-t)_{+}\}\hspace{-1bp}\Big\},\label{eq:CVAR}
\end{equation}
\textit{at confidence level} $\alpha\in(0,1]$. Intuitively, \textbf{\textit{$\mathrm{CV@R}^{\alpha}(Z)$
}}is the \textit{mean of the worst} $\alpha\%$ of the values of $Z$,
and\textbf{\textit{ }}is a strict generalization of expectation; in
particular, it is true that
\begin{align}
\mathrm{CV@R}^{1}(Z) & \equiv\mathbb{E}\{Z\}\le\mathrm{CV@R}^{\alpha}(Z),\forall\alpha\in(0,1],\;\text{and}\\
\mathrm{CV@R}^{0}(Z) & \triangleq\lim_{\alpha\downarrow0}\mathrm{CV@R}^{\alpha}(Z)\equiv\mathrm{ess\hspace{1bp}sup}\,Z.
\end{align}
One of the most important properties of \textbf{\textit{$\mathrm{CV@R}$
}}is that it constitutes a \textit{coherent} risk measure, meaning
that it is a \textit{convex}, \textit{monotone}, \textit{translation
equivariant} and \textit{positively homogeneous} functional of its
argument; see (\citet{ShapiroLectures_2ND}, Section 6.3).

By setting $Z\equiv\ell(f(\boldsymbol{x},\boldsymbol{\theta}),y),\boldsymbol{\theta}\in\mathbb{R}^{m}$,
we may now formulate the \textbf{\textit{$\mathrm{CV@R}$}} statistical
learning problem as
\begin{equation}
\boxed{\inf_{\boldsymbol{\theta}\in\mathbb{R}^{m}}\mathrm{CV@R}_{{\cal P}_{{\cal D}}}^{\alpha}[\ell(f(\boldsymbol{x},\boldsymbol{\theta}),y)].}\label{eq:CVAR_Original}
\end{equation}
Observe that due to its defining properties, the \textbf{\textit{$\mathrm{CV@R}$
}}problem is most intuitive, and allows for an excellent \textit{tunable}
tradeoff between risk neutrality (for $\alpha\equiv1$), and minimax
robustness (as $\alpha\downarrow0$). Additionally, because $\mathrm{CV@R}$
is a coherent risk measure, it follows that problem (\ref{eq:CVAR_Original})
is convex whenever $\ell(f(\boldsymbol{x},\cdot),y)$ is convex for
each $(\boldsymbol{x},y)$, and strongly convex whenever $\ell(f(\boldsymbol{x},\cdot),y)$
is strongly convex for each $(\boldsymbol{x},y)$ \citep{Kalogerias2018b}.
Thus, problem (\ref{eq:CVAR_Original}) is favorably structured.

However, because \textbf{\textit{$\mathrm{CV@R}$}} is itself defined
as the optimal value of a stochastic program, it is difficult to evaluate
analytically, especially in a data-driven setting. Still, we may leverage
the definition of \textbf{\textit{$\mathrm{CV@R}$}} and reformulate
(\ref{eq:CVAR_Original}) as a risk-neutral stochastic program over
\textit{both} variables $(\boldsymbol{\theta},t)$ as
\begin{equation}
\boxed{\inf_{(\boldsymbol{\theta},t)\in\mathbb{R}^{m}\times\mathbb{R}}\mathbb{E}_{{\cal P}_{{\cal D}}}\Big\{ t+\dfrac{1}{\alpha}(\ell(f(\boldsymbol{x},\boldsymbol{\theta}),y)-t)_{+}\Big\}.}\label{eq:CVAR_min}
\end{equation}
Although problem (\ref{eq:CVAR_min}) can now be tackled using standard
methods of stochastic optimization, the structural benefits of the
$\mathrm{CV@R}$ functional are largely gone: For instance, although
it is true that (\ref{eq:CVAR_min}) is convex whenever the composition
$\ell(f(\boldsymbol{x},\cdot),y)$ is convex, it \textit{might} \textit{not}
be strongly convex, even if $\ell(f(\boldsymbol{x},\cdot),y)$ is.
This is important, because it would imply that classical setups, such
as linear least squares, might result in badly behaving $\mathrm{CV@R}$
problems, for $\alpha\in(0,1)$. Of course, those issues can only
get worse in the nonconvex setting, e.g., when the function $f$ is
a Deep Neural Network (DNN).

Nevertheless, it is intuitive that, due to the close relationship
between problems (\ref{eq:CVAR_Original}) and (\ref{eq:CVAR_min}),
the good behavior of the former should carry through to the latter,
and classical solution strategies, such as SGD, should exhibit good
performance. This work shows that this is indeed the case, even in
the nonconvex regime.\vspace{-7bp}

\section{$\boldsymbol{\mathrm{CV@R}}$ Stochastic Gradient Descent\vspace{-6bp}
}

Since the distribution ${\cal P}_{{\cal D}}$ is unknown, the stochastic
program (\ref{eq:RiskNeutral}) (cf. (\ref{eq:CVAR_min})) is impossible
to solve \textit{a priori}. Instead, one should rely on \textit{observable}
example pairs; such empirical data are the only available information
primitives, based on which a near-optimal $f(\cdot,\boldsymbol{\theta}^{*})$
might become possible to discover. Regarding the availability of such
data, there are two distinct settings, the \textit{batch} and and
the \textit{sequential}. The first assumes the availability of a finite
dataset $\{(\boldsymbol{x}^{n},y^{n})\}_{n=0}^{N}$, and replaces
${\cal P}_{{\cal D}}$ in (\ref{eq:RiskNeutral}) (cf. (\ref{eq:CVAR_min}))
with the empirical measure induced by the dataset; in the literature,
this is usually referred to as Empirical ``Risk'' Minimization (ERM)
\citep{Vapnik2000}, and Sample Average Approximation (SAA) \citep{ShapiroLectures_2ND}.
In the second setting, a possibly infinite in length \textit{stream
of data} $\{(\boldsymbol{x}^{n},y^{n})\}_{n=0}^{\infty}$ is available
sequentially (or in sequential batches), and the focus is on solving
(\ref{eq:RiskNeutral}) (cf. (\ref{eq:CVAR_min})) directly, primarily
via stochastic approximation \citep{Kushner2003}. Note that, at least
from the perspective of stochastic optimization, the sequential setting
contains the batch setting as a special, nonetheless important case.

In this paper we are assuming the sequential data setting. This conforms
with countless real-time applications, and is also the standard problem
setup in stochastic optimization. Specifically, we study the standard
stochastic gradient descent algorithm, applied to the equivalent $\mathrm{CV@R}$
problem (\ref{eq:CVAR_min}). Throughout, we make the following essential
but mild assumptions on the composition $\ell(f(\boldsymbol{x},\cdot),y)$.

\begin{assumption}\label{Assumption1} Unless the function $\ell(f(\boldsymbol{x},\cdot),y)$
is convex on $\mathbb{R}^{m}$ for ${\cal P}_{{\cal D}}$-almost all
$(\boldsymbol{x},y)$, then for each $\boldsymbol{\theta}\in\mathbb{R}^{m}$:
\begin{enumerate}
\item $\ell(f(\boldsymbol{x},\cdot),y)$ is $C_{\boldsymbol{\theta}}(\boldsymbol{x},y)$-Lipschitz
on a neighborhood $\boldsymbol{\theta}$ for ${\cal P}_{{\cal D}}$-almost
all $(\boldsymbol{x},y)$, and it is true that $\mathbb{E}_{{\cal P}_{{\cal D}}}\{C_{\boldsymbol{\theta}}(\boldsymbol{x},y)\}<\infty$.
\item \label{Ass12}$\ell(f(\boldsymbol{x},\cdot),y)$ is differentiable
at $\boldsymbol{\theta}$ for ${\cal P}_{{\cal D}}$-almost all $(\boldsymbol{x},y)$,
and ${\cal P}_{{\cal D}}(\ell(f(\boldsymbol{x},\boldsymbol{\theta}),y)=t)\equiv0$
for all $(\boldsymbol{\theta},t)\in\mathbb{R}^{m}\times\mathbb{R}$.
\end{enumerate}
\end{assumption}

For convenience, let us define, for $(\boldsymbol{\theta},t)\in\mathbb{R}^{m}\times\mathbb{R}$,
\begin{equation}
G_{\alpha}(\boldsymbol{\theta},t)\triangleq\mathbb{E}_{{\cal P}_{{\cal D}}}\Big\{ t+\dfrac{1}{\alpha}(\ell(f(\boldsymbol{x},\boldsymbol{\theta}),y)-t)_{+}\Big\}.
\end{equation}
Then it may be shown that, under Assumption \ref{Assumption1}, differentiation
may be interchanged with expectation for $G_{\alpha}$ (\citep{ShapiroLectures_2ND},
Section 7.2.4), yielding, for every $(\boldsymbol{\theta},t)$, the
(sub)gradient representation\renewcommand{\arraystretch}{2}
\begin{equation}
\hspace{-1bp}\hspace{-1bp}\hspace{-1bp}\hspace{-1bp}\hspace{-1bp}\hspace{-1bp}\hspace{-1bp}\hspace{-0.5bp}\nabla G_{\alpha}(\boldsymbol{\theta},t)\hspace{-1bp}\hspace{-0.5bp}=\hspace{-1bp}\hspace{-1bp}\hspace{-1bp}\hspace{-1bp}\begin{bmatrix}\dfrac{1}{\alpha}\mathbb{E}_{{\cal P}_{{\cal D}}}\{{\bf 1}_{{\cal A}(\boldsymbol{\theta},t)}(\boldsymbol{x},y)\nabla_{\boldsymbol{\theta}}\ell(f(\boldsymbol{x},\boldsymbol{\theta}),y)\}\\
-\dfrac{1}{\alpha}\mathbb{E}_{{\cal P}_{{\cal D}}}\{{\bf 1}_{{\cal A}(\boldsymbol{\theta},t)}(\boldsymbol{x},y)\}+1
\end{bmatrix}\hspace{-1bp}\hspace{-1bp}\hspace{-0.5bp},\hspace{-1bp}\hspace{-1bp}\hspace{-1bp}\hspace{-1bp}\hspace{-1bp}\label{eq:GRAD}
\end{equation}
where for brevity and for later use we have defined the \textit{event-valued}
multifunction ${\cal A}:\mathbb{R}^{m}\times\mathbb{R}\rightrightarrows{\cal D}$
as
\begin{equation}
{\cal A}(\boldsymbol{\theta},t)\triangleq\{(\boldsymbol{x},y)\in{\cal D}|\ell(f(\boldsymbol{x},\boldsymbol{\theta}),y)-t>0\},
\end{equation}
for $(\boldsymbol{\theta},t)\in\mathbb{R}^{m}\times\mathbb{R}$. We
note that, for each $(\boldsymbol{\theta},t)$, the set ${\cal A}(\boldsymbol{\theta},t)$
contains all examples corresponding to the \textit{positive section}
of the function $\ell(f(\bullet,\boldsymbol{\theta}),\cdot)-t$.

Leveraging (\ref{eq:GRAD}), and given an independent and identically
distributed datastream $\{(\boldsymbol{x}^{n},y^{n})\}_{n=0}^{\infty}$,
we can now outline the simplest and most obvious scheme for possibly
tackling the $\mathrm{CV@R}$ problem (\ref{eq:CVAR_min}), i.e.,
the standard SGD rule, described via the recursive updates
\begin{flalign}
t^{n+1} & =t^{n}-\gamma\Big[1-\dfrac{1}{\alpha}{\bf 1}_{{\cal A}(\boldsymbol{\theta}^{n},t^{n})}(\boldsymbol{x}^{n+1},y^{n+1})\Big]\quad\text{and}\label{eq:SGD_1}\\
\boldsymbol{\theta}^{n+1} & =\boldsymbol{\theta}^{n}-\beta\dfrac{1}{\alpha}{\bf 1}_{{\cal A}(\boldsymbol{\theta}^{n},t^{n})}(\boldsymbol{x}^{n+1},y^{n+1})\nabla_{\boldsymbol{\theta}}\ell(f(\boldsymbol{x}^{n+1},\boldsymbol{\theta}^{n}),y^{n+1}),\label{eq:SGD_2}
\end{flalign}
where $n\in\mathbb{N}$ is an iteration index, $\beta>0$ and $\gamma>0$
are constant stepsizes, and where $(\boldsymbol{\theta}^{0},t^{0})$
are appropriately chosen initial values.

We observe that the SGD updates (\ref{eq:SGD_1}) and (\ref{eq:SGD_2})
can be regarded as a modification of the standard risk-neutral SGD
(solving (\ref{eq:RiskNeutral})), but where learning happens \textit{if
and only if }$\ell(f(\boldsymbol{x}^{n+1},\boldsymbol{\theta}^{n}),y^{n+1})-t^{n}\ge0$,
for each $n$. The update in $t$ controls the frequency of learning,
as well as the proportion of examples that participate in learning.
Also note that if $\alpha\equiv1$, then $t^{n}$ is nonincreasing,
and therefore $\boldsymbol{\theta}^{n}$ should approach a risk-neutral
solution. In the following, we suggestively refer to the algorithm
comprised by (\ref{eq:SGD_1}) and (\ref{eq:SGD_2}) as $\mathrm{CV@R}$-SGD.\vspace{-7bp}

\section{\label{sec:PL}Polyak-\L ojasiewicz Conditions\vspace{-6bp}
}

We next present the standard Polyak-\L ojasiewicz (P\L ) inequality,
first appeared in \citep{Polyak1963}.
\begin{defn}
\textbf{(P\L{} \citet{Polyak1963})}\label{def:PL} We say that a
function $\varphi:\mathbb{R}^{L}\rightarrow\mathbb{R}$ satisfies
the \textit{Polyak-\L ojasiewicz (P\L ) inequality} \textit{with parameter}
$\mu>0$ on $\Sigma\subseteq\mathbb{R}^{L}$, if and only if $\varphi$
is differentiable on $\Sigma$ and, for every $\boldsymbol{x}\in\Sigma$,
\begin{equation}
\dfrac{1}{2}\Vert\nabla\varphi(\boldsymbol{x})\Vert_{2}^{2}\ge\mu(\varphi(\boldsymbol{x})-\varphi^{\star}),
\end{equation}
where $\varphi^{\star}\triangleq\inf_{\boldsymbol{x}\in\Sigma}\varphi(\boldsymbol{x})$.
\end{defn}
In a recent seminal article \citep{Karimi2016}, the P\L{} inequality
was exploited to show linear convergence of gradient methods under
multiple interesting and useful setups. Further, \citep{Karimi2016}
shows that strong convexity implies the P\L{} inequality, but also
that there are lots of \textit{nonconvex} functions obeying the P\L{}
inequality. This indeed implies that S(GD) converges \textit{globally
and linearly }for such functions.

For our purposes, unfortunately, the standard P\L{} inequality (Definition
\ref{def:PL}) will not suffice. Instead, we introduce and rely on
a generalization, which we call the \textit{set-restricted P\L{} inequality},
as follows.
\begin{defn}
\textbf{(Set-Restricted P\L )}\label{def:Set-Restricted-PL-1} Consider
a measurable function $\varphi:\mathbb{R}^{L}\times\mathbb{R}^{M}\rightarrow\mathbb{R}$,
a Borel-valued multifunction ${\cal B}:\mathbb{R}^{L}\rightrightarrows\mathbb{R}^{M}$,
and a probability measure ${\cal M}$ on $\mathscr{B}(\mathbb{R}^{M})$.
We say that $\varphi$ satisfies the (diagonal) ${\cal B}$\textit{-restricted
Polyak-\L ojasiewicz (P\L ) inequality} \textit{with parameter} $\mu>0$,
relative to ${\cal M}$ and on a subset $\Sigma\subseteq\mathbb{R}^{L}$,
if and only if $\varphi(\cdot,\boldsymbol{w})$ is subdifferentiable
on $\Sigma$ for ${\cal M}$-almost every $\boldsymbol{w}\in\mathbb{R}^{M}$,
and it is true that, for every $\boldsymbol{z}\in\Sigma$,
\begin{equation}
\dfrac{1}{2}\Vert\mathbb{E}_{{\cal M}}\{\nabla_{\boldsymbol{z}}\varphi(\boldsymbol{z},\boldsymbol{w})|{\cal B}(\boldsymbol{z})\}\Vert_{2}^{2}\ge\mu\mathbb{E}_{{\cal M}}\{\varphi(\boldsymbol{z},\boldsymbol{w})-\varphi^{\star}(\boldsymbol{z})|{\cal B}(\boldsymbol{z})\},
\end{equation}
where $\varphi^{\star}(\cdot)\triangleq\inf_{\widetilde{\boldsymbol{z}}\in\Sigma}\mathbb{E}_{{\cal M}}\{\varphi(\widetilde{\boldsymbol{z}},\boldsymbol{w})|{\cal B}(\cdot)\}$.
\end{defn}
Although admittedly somewhat mysterious at first sight, the set-restricted
P\L{} inequality is essentially the same as the classical P\L{} inequality
as considered for standard stochastic optimization \citep{Karimi2016},
with the important difference that expectation is replaced by conditional
expectation relative to an event \textit{varying} in the argument
of the function involved (i.e., an event-valued multifunction). From
a learning perspective, the set-restricted P\L{} inequality quantifies
the curvature of the loss surface by restricting attention on sets
of learning examples that matter (in Definition \ref{def:Set-Restricted-PL-1},
${\cal B}$ plays this role).

One fact revealing the importance of the set-restricted P\L{} inequality
of Definition \ref{def:Set-Restricted-PL-1} is that it is satisfied
by all smooth and strongly convex losses. In particular, we have the
following result.
\begin{prop}
\textbf{\textup{(Strong Convexity $\implies$ Set-Restricted P\L )}}\label{prop:SC_sat_scPL}
Suppose that the loss $\ell(f(\boldsymbol{x},\cdot),y)$ is $L$-smooth
and $\mu$-strongly convex for ${\cal P}_{{\cal D}}$-almost all
$(\boldsymbol{x},y)$. Then, for every pair $(\boldsymbol{\theta},{\cal B})\in\mathbb{R}^{m}\times\mathscr{B}({\cal D})$
such that ${\cal P}_{{\cal D}}({\cal B})>0$, it is true that
\begin{equation}
\dfrac{1}{2}\Vert\mathbb{E}\{\nabla_{\boldsymbol{\theta}}\ell(f(\boldsymbol{x},\boldsymbol{\theta}),y)|{\cal B})\}\Vert_{2}^{2}\ge\mu\mathbb{E}\{\ell(f(\boldsymbol{x},\boldsymbol{\theta}),y)-\ell^{\star}({\cal B})|{\cal B}\},
\end{equation}
where $\ell^{\star}({\cal B})\equiv\inf_{\widetilde{\boldsymbol{\theta}}}\mathbb{E}\{\ell(f(\boldsymbol{x},\widetilde{\boldsymbol{\theta}}),y)|{\cal B}\}$.
\end{prop}
\begin{proof}[Proof of Proposition \ref{prop:SC_sat_scPL}]
Taking conditional (rescaled) expectations relative to ${\cal B}$,
we get that, for every qualifying pair $(\boldsymbol{\theta},\boldsymbol{\theta}')$,
\begin{equation}
\mathbb{E}\{\ell(f(\boldsymbol{x},\boldsymbol{\theta}),y)|{\cal B}\}\ge\mathbb{E}\{\ell(f(\boldsymbol{x},\boldsymbol{\theta}'),y)|{\cal B}\}+\langle\mathbb{E}\{\nabla_{\boldsymbol{\theta}}\ell(f(\boldsymbol{x},\boldsymbol{\theta}'),y)|{\cal B}\},\boldsymbol{\theta}-\boldsymbol{\theta}'\rangle+\dfrac{\mu}{2}\Vert\boldsymbol{\theta}-\boldsymbol{\theta}'\Vert_{2}^{2}.
\end{equation}
By Assumption \ref{Assumption1}, we may interchange expectation with
differentiation, further obtaining
\begin{equation}
L_{{\cal B}}(\boldsymbol{\theta})\ge L_{{\cal B}}(\boldsymbol{\theta}')+\langle\nabla L_{{\cal B}}(\boldsymbol{\theta}),\boldsymbol{\theta}-\boldsymbol{\theta}'\rangle+\dfrac{\mu}{2}\Vert\boldsymbol{\theta}-\boldsymbol{\theta}'\Vert_{2}^{2},\quad\forall(\boldsymbol{\theta},\boldsymbol{\theta}'),
\end{equation}
where $L_{{\cal B}}(\cdot)\triangleq\mathbb{E}\{\ell(f(\boldsymbol{x},\cdot),y)|{\cal B}\}$.
This shows that the restricted expected loss $L_{{\cal B}}$ is $\mu$-strongly
convex. In exactly the same fashion, it follows that $L_{{\cal B}}$
is $L$-smooth, as well. Consequently, $L_{{\cal B}}$ satisfies the
P\L{} inequality with parameter $\mu$ \citep{Karimi2016}, i.e.,
it is true that, for every qualifying $\boldsymbol{\theta}$,
\begin{equation}
\dfrac{1}{2}\Vert\nabla L_{{\cal B}}(\boldsymbol{\theta})\Vert_{2}^{2}\ge\mu(L_{{\cal B}}(\boldsymbol{\theta})-{\textstyle \inf_{\boldsymbol{\theta}}L_{{\cal B}}(\boldsymbol{\theta})}).
\end{equation}
But $\nabla L_{{\cal B}}(\cdot)\equiv\mathbb{E}\{\nabla_{\boldsymbol{\theta}}\ell(f(\boldsymbol{x},\cdot),y)|{\cal B}\}$.
Enough said.
\end{proof}
From Proposition \ref{prop:SC_sat_scPL}, it follows that every smooth
strongly convex loss satisfies the set-restricted P\L{} inequality
relative to any qualifying event-valued multifunction of choice. For
instance, in the notation of Proposition \ref{prop:SC_sat_scPL},
one may set ${\cal B}\equiv{\cal A}(\boldsymbol{\theta},t)$, for
every fixed pair $(\boldsymbol{\theta},t)$. This choice is particularly
important, as we will see in the next section.

\section{Linear Convergence of $\boldsymbol{\mathrm{CV@R}}$-SGD}

In this section, we present the main results of the paper. We start
by showing that, quite interestingly, if the loss satisfies the set-restricted
P\L{} inequality relative to the multifunction ${\cal A}$, then the
objective function $G_{\alpha}$ satisfies the ordinary P\L{} inequality.
The relevant result follows.
\begin{lem}
\textbf{\textup{($G$ is Polyak-\L ojasiewicz)}}\label{lem:CV@R_Approx}
Fix an $\alpha\in(0,1)$ and consider a set $\Delta\triangleq\Delta_{m}\times(-\infty,\overline{t}]\subseteq\mathbb{R}^{m}\times\mathbb{R}$,
for which the following are in effect:
\begin{enumerate}
\item ${\textstyle \arg\min}_{\Delta}G_{\alpha}(\boldsymbol{\theta},t)\neq\emptyset$,
with $(\boldsymbol{\theta}^{*},t^{*})$ being an arbitrary member
of this set.
\item the random loss $\ell(f(\boldsymbol{x},\cdot),y)$ satisfies the ${\cal A}$-restricted
P\L{} inequality with parameter $\mu>0$, relative to ${\cal D}$
and on $\Delta$, i.e.,
\begin{equation}
\dfrac{1}{2}\Vert\mathbb{E}\{\nabla_{\boldsymbol{\theta}}\ell(f(\boldsymbol{x},\boldsymbol{\theta}),y)|{\cal A}(\boldsymbol{\theta},t)\}\Vert_{2}^{2}\ge\mu\mathbb{E}\{\ell(f(\boldsymbol{x},\boldsymbol{\theta}),y)-\ell^{\star}(\boldsymbol{\theta},t)|{\cal A}(\boldsymbol{\theta},t)\},
\end{equation}
for all $(\boldsymbol{\theta},t)\in\Delta$, where $\ell^{\star}(\bullet,\cdot)\equiv\inf_{\widetilde{\boldsymbol{\theta}}\in\Delta_{m}}\mathbb{E}\{\ell(f(\boldsymbol{x},\widetilde{\boldsymbol{\theta}}),y)|{\cal A}(\bullet,\cdot)\}$.
\end{enumerate}
Then, for any subset $\Delta'\subseteq\Delta$ such that
\begin{equation}
{\cal P}_{{\cal D}}({\cal A}(\boldsymbol{\theta},t))>\alpha+2\alpha\mu(t^{*}-t)_{+},\quad\forall(\boldsymbol{\theta},t)\in\Delta',
\end{equation}
the $\mathrm{CV@R}$ objective $G_{\alpha}$ obeys
\begin{equation}
\mu(G_{\alpha}(\boldsymbol{\theta},t)-G_{\alpha}(\boldsymbol{\theta}^{*},t^{*}))\le\dfrac{1}{2}\Vert\nabla G(\boldsymbol{\theta},t)\Vert_{2}^{2},
\end{equation}
 everywhere on $\Delta'$.
\end{lem}
\begin{proof}[Proof of Lemma \ref{lem:CV@R_Approx}]
For every $(\boldsymbol{x},y)$, we have
\begin{flalign}
 & \hspace{-1bp}\hspace{-1bp}\hspace{-1bp}\hspace{-1bp}\hspace{-1bp}\hspace{-1bp}\hspace{-1bp}\hspace{-1bp}\hspace{-1bp}g_{(\boldsymbol{x},y)}^{\alpha}(\boldsymbol{\theta},t)-g_{(\boldsymbol{x},y)}^{\alpha}(\boldsymbol{\theta}^{*},t^{*})\nonumber \\
 & \equiv t-t^{*}+\dfrac{1}{\alpha}(\ell(f(\boldsymbol{x},\boldsymbol{\theta}),y)-t)_{+}-\dfrac{1}{\alpha}(\ell(f(\boldsymbol{x},\boldsymbol{\theta}^{*}),y)-t^{*})_{+}\nonumber \\
 & \le t-t^{*}+\dfrac{1}{\alpha}(\ell(f(\boldsymbol{x},\boldsymbol{\theta}),y)-t)_{+}-\dfrac{1}{\alpha}{\bf 1}_{{\cal A}(\boldsymbol{\theta},t)}(\boldsymbol{x},y)(\ell(f(\boldsymbol{x},\boldsymbol{\theta}^{*}),y)-t^{*})\nonumber \\
 & \equiv t-t^{*}+\dfrac{1}{\alpha}{\bf 1}_{{\cal A}(\boldsymbol{\theta},t)}(\boldsymbol{x},y)\big[\ell(f(\boldsymbol{x},\boldsymbol{\theta}),y)-\ell(f(\boldsymbol{x},\boldsymbol{\theta}^{*}),y)+t^{*}-t\big]\nonumber \\
 & =(t^{*}-t)\Big(\dfrac{1}{\alpha}{\bf 1}_{{\cal A}(\boldsymbol{\theta},t)}(\boldsymbol{x},y)-1\Big)+\dfrac{1}{\alpha}{\bf 1}_{{\cal A}(\boldsymbol{\theta},t)}(\boldsymbol{x},y)\big[\ell(f(\boldsymbol{x},\boldsymbol{\theta}),y)-\ell(f(\boldsymbol{x},\boldsymbol{\theta}^{*}),y)\big].
\end{flalign}
By taking expectation on both sides, it follows that
\begin{flalign}
 & \hspace{0.5bp}G_{\alpha}(\boldsymbol{\theta},t)-G_{\alpha}(\boldsymbol{\theta}^{*},t^{*})\nonumber \\
 & \le\hspace{-1bp}\hspace{-1bp}(t^{*}\hspace{-1bp}-\hspace{-1bp}t)\Big(\dfrac{1}{\alpha}{\cal P}_{{\cal D}}({\cal A}(\boldsymbol{\theta},t))\hspace{-1bp}-\hspace{-1bp}1\Big)\hspace{-1bp}\hspace{-0.5bp}+\hspace{-0.5bp}\hspace{-1bp}\dfrac{1}{\alpha}\mathbb{E}\{{\bf 1}_{{\cal A}(\boldsymbol{\theta},t)}(\boldsymbol{x},y)[\ell(f(\boldsymbol{x},\boldsymbol{\theta}),y)\hspace{-1bp}-\hspace{-1bp}\ell(f(\boldsymbol{x},\boldsymbol{\theta}^{*}),y)]\}\nonumber \\
 & \equiv\hspace{-1bp}\hspace{-1bp}(t^{*}\hspace{-1bp}-\hspace{-1bp}t)\Big(\dfrac{1}{\alpha}{\cal P}_{{\cal D}}({\cal A}(\boldsymbol{\theta},t))\hspace{-1bp}-\hspace{-1bp}1\Big)\hspace{-1bp}\hspace{-0.5bp}+\hspace{-0.5bp}\hspace{-1bp}\dfrac{1}{\alpha}\mathbb{E}\big\{\ell(f(\boldsymbol{x},\boldsymbol{\theta}),y)\hspace{-1bp}-\hspace{-1bp}\ell(f(\boldsymbol{x},\boldsymbol{\theta}^{*}),y)|{\cal A}(\boldsymbol{\theta},t)\big\}{\cal P}_{{\cal D}}({\cal A}(\boldsymbol{\theta},t))\nonumber \\
 & \equiv\hspace{-1bp}\hspace{-1bp}(t^{*}\hspace{-1bp}-\hspace{-1bp}t)\Big(\dfrac{1}{\alpha}{\cal P}_{{\cal D}}({\cal A}(\boldsymbol{\theta},t))\hspace{-1bp}-\hspace{-1bp}1\Big)\hspace{-1bp}\hspace{-0.5bp}+\hspace{-0.5bp}\hspace{-1bp}\dfrac{1}{\alpha}(\mathbb{E}\big\{\ell(f(\boldsymbol{x},\boldsymbol{\theta}),y)|{\cal A}(\boldsymbol{\theta},t)\big\}\hspace{-1bp}-\hspace{-1bp}\mathbb{E}\big\{\ell(f(\boldsymbol{x},\boldsymbol{\theta}^{*}),y)|{\cal A}(\boldsymbol{\theta},t)\big\}){\cal P}_{{\cal D}}({\cal A}(\boldsymbol{\theta},t))\nonumber \\
 & \le\hspace{-1bp}\hspace{-1bp}(t^{*}\hspace{-1bp}-\hspace{-1bp}t)\Big(\dfrac{1}{\alpha}{\cal P}_{{\cal D}}({\cal A}(\boldsymbol{\theta},t))\hspace{-1bp}-\hspace{-1bp}1\Big)\hspace{-1bp}\hspace{-0.5bp}+\hspace{-0.5bp}\hspace{-1bp}\dfrac{1}{\alpha}(\mathbb{E}\big\{\ell(f(\boldsymbol{x},\boldsymbol{\theta}),y)|{\cal A}(\boldsymbol{\theta},t)\big\}\hspace{-1bp}-\hspace{-1bp}\ell^{\star}(\boldsymbol{\theta},t)\big\}){\cal P}_{{\cal D}}({\cal A}(\boldsymbol{\theta},t))\nonumber \\
 & \equiv\hspace{-1bp}\hspace{-1bp}(t^{*}\hspace{-1bp}-\hspace{-1bp}t)\Big(\dfrac{1}{\alpha}{\cal P}_{{\cal D}}({\cal A}(\boldsymbol{\theta},t))\hspace{-1bp}-\hspace{-1bp}1\Big)\hspace{-1bp}\hspace{-0.5bp}+\hspace{-0.5bp}\hspace{-1bp}\dfrac{1}{\alpha}(\mathbb{E}\big\{\ell(f(\boldsymbol{x},\boldsymbol{\theta}),y)\hspace{-1bp}-\hspace{-1bp}\ell^{\star}(\boldsymbol{\theta},t)|{\cal A}(\boldsymbol{\theta},t)\big\}){\cal P}_{{\cal D}}({\cal A}(\boldsymbol{\theta},t))
\end{flalign}
Therefore, from the set-restricted P\L{} inequality we get
\begin{flalign}
G_{\alpha}(\boldsymbol{\theta},t)-G_{\alpha}(\boldsymbol{\theta}^{*},t^{*}) & \le(t^{*}-t)\Big(\dfrac{1}{\alpha}{\cal P}_{{\cal D}}({\cal A}(\boldsymbol{\theta},t))-1\Big)\nonumber \\
 & \quad\quad+\dfrac{1}{2\mu\alpha}\Vert\mathbb{E}\{\nabla_{\boldsymbol{\theta}}\ell(f(\boldsymbol{x},\boldsymbol{\theta}),y)|{\cal A}(t,\boldsymbol{\theta})\}\Vert_{2}^{2}{\cal P}_{{\cal D}}({\cal A}(\boldsymbol{\theta},t)).
\end{flalign}
Next, assuming that
\begin{alignat}{2}
{\cal P}_{{\cal D}}({\cal A}(\boldsymbol{\theta},t)) & >\alpha+\alpha2\mu(t^{*}-t)_{+} & \implies & {\cal P}_{{\cal D}}({\cal A}(\boldsymbol{\theta},t))>\alpha\iff\dfrac{1}{\alpha}{\cal P}_{{\cal D}}({\cal A}(\boldsymbol{\theta},t))-1>0\\
 & \ge\alpha+\alpha2\mu(t^{*}-t) & \iff & (t^{*}-t)\le\dfrac{1}{2\mu}\Big(\dfrac{1}{\alpha}{\cal P}_{{\cal D}}({\cal A}(\boldsymbol{\theta},t))-1\Big)>0
\end{alignat}
for all $(\boldsymbol{\theta},t)$ in a subset $\Delta'\subseteq\Delta$,
it follows that
\begin{equation}
(t^{*}-t)\Big(\dfrac{1}{\alpha}{\cal P}_{{\cal D}}({\cal A}(\boldsymbol{\theta},t))-1\Big)\le\dfrac{1}{2\mu}\Big(1-\dfrac{1}{\alpha}{\cal P}_{{\cal D}}({\cal A}(\boldsymbol{\theta},t))\Big)^{2},
\end{equation}
for all $(\boldsymbol{\theta},t)$ on that subset. Therefore, we
may further write
\begin{flalign}
G_{\alpha}(\boldsymbol{\theta},t)-G_{\alpha}(\boldsymbol{\theta}^{*},t^{*}) & \le\dfrac{1}{2\mu}\Big(\dfrac{1}{\alpha}{\cal P}_{{\cal D}}({\cal A}(\boldsymbol{\theta},t))-1\Big)^{2}\nonumber \\
 & \quad\quad+\dfrac{1}{2\mu\alpha^{2}}\Vert\mathbb{E}\{\nabla_{\boldsymbol{\theta}}\ell(f(\boldsymbol{x},\boldsymbol{\theta}),y)|{\cal A}(t,\boldsymbol{\theta})\}\Vert_{2}^{2}({\cal P}_{{\cal D}}({\cal A}(\boldsymbol{\theta},t)))^{2}.
\end{flalign}
Now, observe that
\begin{equation}
\nabla G_{\alpha}(\boldsymbol{\theta},t)=\begin{bmatrix}\dfrac{1}{\alpha}\mathbb{E}\{{\bf 1}_{{\cal A}(\boldsymbol{\theta},t)}(\boldsymbol{x},y)\nabla_{\boldsymbol{\theta}}\ell(f(\boldsymbol{x},\boldsymbol{\theta}),y)\}\\
1-\dfrac{1}{\alpha}{\cal P}_{{\cal D}}({\cal A}(\boldsymbol{\theta},t))
\end{bmatrix},
\end{equation}
from where we immediately deduce that, for every $(\boldsymbol{\theta},t)\in\Delta'$,
\begin{equation}
\mu(G_{\alpha}(\boldsymbol{\theta},t)-G_{\alpha}(\boldsymbol{\theta}^{*},t^{*}))\le\dfrac{1}{2}\Vert\nabla G(\boldsymbol{\theta},t)\Vert_{2}^{2},
\end{equation}
and the proof is complete.
\end{proof}
In what follows, let $\{\mathscr{D}_{n}\}_{n\in\mathbb{N}}$ be the
history (i.e., filtration) generated by $\mathrm{CV@R}$-$\mathrm{SGD}$
and the observables (i.e., available datastream). Our main result
follows, showing linear convergence of $\mathrm{CV@R}$-SGD under
the set-restricted P\L{} inequality.
\begin{thm}
\textbf{\textup{(Linear Convergence of $\boldsymbol{\mathrm{CV@R}}$-SGD)}}\label{thm:CV@R_Linear}
Fix $\alpha\in(0,1)$, let Assumption \ref{Assumption1} be in effect
and suppose that, for a subset $\Delta\equiv\Delta{}_{m}\times[-\infty,\overline{t}]$,
with $\Delta_{m}\subseteq\mathbb{R}^{m}$, conditions (1)-(2) of Lemma
\ref{lem:CV@R_Approx} are in effect, as well. Further, for fixed
$T\in\mathbb{N}$, let $\gamma$ be small enough such that
\begin{equation}
\mathbb{E}_{n}\{t^{n+1}|\mathscr{D}_{n}\}>t^{n}+2\gamma\mu(t^{*}-t^{n})_{+},\quad\forall n\in\mathbb{N}_{T},\label{eq:recursive}
\end{equation}
and let $\Delta_{T}\triangleq\{\boldsymbol{\theta}^{n},t^{n}\}_{n\in\mathbb{N}_{T}}$
be the set of points generated by $\mathrm{CV@R}$-$\mathrm{SGD}$.
As long as $\Delta_{T}\subseteq\Delta$ (in the notation of Lemma
\ref{lem:CV@R_Approx}), $G_{\alpha}$ is $L$-smooth on $\Delta_{T}$,
and $2\mu\min\{\beta,\gamma\}<1$, it is true that
\begin{equation}
\hspace{-1bp}\hspace{-1bp}\hspace{-1bp}\hspace{-1bp}\hspace{-1bp}\hspace{-1bp}\hspace{-1bp}\hspace{-1bp}\hspace{-1bp}\hspace{-1bp}\hspace{-1bp}\hspace{-1bp}\hspace{-1bp}\hspace{-1bp}\hspace{-1bp}\hspace{-1bp}\hspace{-1bp}\hspace{-1bp}\boxed{\begin{array}{l}
\vspace{-35.5pt}\\
\hspace{-1bp}\hspace{-1bp}\hspace{-1bp}\hspace{-1bp}\mathbb{E}\big\{ G_{\alpha}(\boldsymbol{\theta}^{T+1},t^{T+1})-G_{\alpha}(\boldsymbol{\theta}^{*},t^{*})\hspace{-1bp}\big\}\\
\quad\le(1-2\mu\min\{\beta,\gamma\})^{T}(G_{\alpha}(\boldsymbol{\theta}^{0},t^{0})-G_{\alpha}(\boldsymbol{\theta}^{*},t^{*}))+\dfrac{(\max\{\beta,\gamma\})^{2}}{\min\{\beta,\gamma\}}\dfrac{L(1+C_{T}^{2})}{4\alpha^{2}\mu},\hspace{-1bp}\hspace{-1bp}\hspace{-1bp}\hspace{-1bp}
\end{array}}
\end{equation}
where $\sup_{n\in\mathbb{N}_{T}}\mathbb{E}\{\Vert\nabla_{\boldsymbol{\theta}}\ell(f(\boldsymbol{x}^{n+1},\boldsymbol{\theta}^{n}),y^{n+1})\Vert_{2}^{2}\}\le C_{T}^{2}$,
and where $(\boldsymbol{\theta}^{*},t^{*})\in{\textstyle \arg\min}_{\Delta}G_{\alpha}(\boldsymbol{\theta},t)$.
\end{thm}
\begin{proof}[Proof of Theorem \ref{thm:CV@R_Linear}]
By the assumptions of the theorem, the elements of $\Delta_{T}$
must satisfy the recursion
\begin{equation}
\mathbb{E}_{n}\{t^{n+1}|\mathscr{D}_{n}\}=t^{n}-\gamma\Big[1-\dfrac{1}{\alpha}{\cal P}_{{\cal D}}({\cal A}(\boldsymbol{\theta}^{n},t^{n}))\Big],\quad n\in\mathbb{N}_{T}^{+}.
\end{equation}
By (\ref{eq:recursive}), we get
\begin{equation}
\mathbb{E}_{n}\{t^{n+1}|\mathscr{D}_{n}\}-t^{n}>2\gamma\mu(t^{*}-t^{n})_{+}\iff\alpha\dfrac{\mathbb{E}_{n}\{t^{n+1}|\mathscr{D}_{n}\}-t^{n}}{\gamma}+\alpha>\alpha+2\alpha\mu(t^{*}-t^{n})_{+},
\end{equation}
or equivalently,
\begin{equation}
{\cal P}_{{\cal D}}({\cal A}(\boldsymbol{\theta}^{n},t^{n}))>\alpha+2\alpha\mu(t^{*}-t^{n})_{+},\quad\forall n\in\mathbb{N}_{T}^{+}.
\end{equation}
Therefore, we may now invoke Lemma \ref{lem:CV@R_Approx}. Indeed,
assuming that $\Delta_{T}\subseteq\Delta$ and that $G_{\alpha}$
is $L$-smooth on $\Delta_{T}$, we may use the $\mathrm{CV@R}$-$\mathrm{SGD}$
updates to write 
\begin{flalign}
G_{\alpha}(\boldsymbol{\theta}^{n+1},t^{n+1}) & \le G_{\alpha}(\boldsymbol{\theta}^{n},t^{n})-\Big\langle\nabla G_{\alpha}(\boldsymbol{\theta}^{n},t^{n}),[\beta{\bf 1}_{m}\,\gamma]^{\boldsymbol{T}}\circ\nabla g_{(\boldsymbol{x}^{n+1},y^{n+1})}^{\alpha}(\boldsymbol{\theta}^{n},t^{n})\hspace{-1bp}\Big\rangle\nonumber \\
 & \quad\quad+\dfrac{L}{2}\Vert[\beta{\bf 1}_{m}\,\gamma]^{\boldsymbol{T}}\circ\nabla g_{(\boldsymbol{x}^{n+1},y^{n+1})}^{\alpha}(\boldsymbol{\theta}^{n},t^{n})\Vert_{2}^{2}
\end{flalign}
for each $n\in\mathbb{N}_{T}$, where ``$\circ$'' denotes the Hadamard
product. Taking expectations relative to $\mathscr{D}_{n}$, we obtain
\begin{flalign}
\mathbb{E}\{G_{\alpha}(\boldsymbol{\theta}^{n+1},t^{n+1})|\mathscr{D}_{n}\} & \le G_{\alpha}(\boldsymbol{\theta}^{n},t^{n})-\langle\nabla G_{\alpha}(\boldsymbol{\theta}^{n},t^{n}),[\beta{\bf 1}_{m}\,\gamma]^{\boldsymbol{T}}\circ\nabla G_{\alpha}(\boldsymbol{\theta}^{n},t^{n})\rangle\nonumber \\
 & \quad\quad+\dfrac{L}{2}\mathbb{E}\Big\{\Vert[\beta{\bf 1}_{m}\,\gamma]^{\boldsymbol{T}}\circ\nabla g_{(\boldsymbol{x}^{n+1},y^{n+1})}^{\alpha}(\boldsymbol{\theta}^{n},t^{n})\Vert_{2}^{2}|\mathscr{D}_{n}\Big\}\nonumber \\
 & \le G_{\alpha}(\boldsymbol{\theta}^{n},t^{n})-\min\{\beta,\gamma\}\Vert\nabla G_{\alpha}(\boldsymbol{\theta}^{n},t^{n})\Vert_{2}^{2}\nonumber \\
 & \quad\quad+\dfrac{L}{2}(\max\{\beta,\gamma\})^{2}\mathbb{E}\{\Vert\nabla g_{(\boldsymbol{x}^{n+1},y^{n+1})}^{\alpha}(\boldsymbol{\theta}^{n},t^{n})\Vert_{2}^{2}|\mathscr{D}_{n}\},
\end{flalign}
By applying Lemma \ref{lem:CV@R_Approx} for $G_{\alpha}$, and using
the fact that
\begin{flalign}
\big\Vert\nabla g_{(\boldsymbol{x}^{n+1},y^{n+1})}^{\alpha}(\boldsymbol{\theta}^{n},t^{n})\hspace{-0.5bp}\hspace{-1bp}\big\Vert_{2}^{2} & \equiv\Big(1-\dfrac{1}{\alpha}{\bf 1}_{{\cal A}(\boldsymbol{\theta}^{n},t^{n})}(\boldsymbol{x}^{n+1},y^{n+1})\Big)^{2}\nonumber \\
 & \quad\quad+\dfrac{1}{\alpha^{2}}{\bf 1}_{{\cal A}(\boldsymbol{\theta}^{n},t^{n})}\Vert\nabla_{\boldsymbol{\theta}}\ell(f(\boldsymbol{x}^{n+1},\boldsymbol{\theta}^{n}),y^{n+1})\Vert_{2}^{2}\nonumber \\
 & \le\max\Big\{1,\Big(\dfrac{1-\alpha}{\alpha}\Big)^{2}\Big\}\nonumber \\
 & \quad\quad+\dfrac{1}{\alpha^{2}}\Vert\nabla_{\boldsymbol{\theta}}\ell(f(\boldsymbol{x}^{n+1},\boldsymbol{\theta}^{n}),y^{n+1})\Vert_{2}^{2}\nonumber \\
 & \le\dfrac{1}{\alpha^{2}}+\dfrac{1}{\alpha^{2}}\Vert\nabla_{\boldsymbol{\theta}}\ell(f(\boldsymbol{x}^{n+1},\boldsymbol{\theta}^{n}),y^{n+1})\Vert_{2}^{2},
\end{flalign}
we further get
\begin{flalign}
\mathbb{E}\{G_{\alpha}(\boldsymbol{\theta}^{n+1},t^{n+1})|\mathscr{D}_{n}\} & \le G_{\alpha}(\boldsymbol{\theta}^{n},t^{n})-\min\{\beta,\gamma\}2\mu(G_{\alpha}(\boldsymbol{\theta}^{n},t^{n})-G_{\alpha}(\boldsymbol{\theta}^{*},t^{*}))\nonumber \\
 & \quad\quad+\dfrac{L}{2}(\max\{\beta,\gamma\})^{2}\dfrac{1+\mathbb{E}\{\Vert\nabla_{\boldsymbol{\theta}}\ell(f(\boldsymbol{x}^{n+1},\boldsymbol{\theta}^{n}),y^{n+1})\Vert_{2}^{2}|\mathscr{D}_{n}\}}{\alpha^{2}},
\end{flalign}
Rearranging and taking expectation one more time, it follows that
\begin{flalign}
\mathbb{E}\{G_{\alpha}(\boldsymbol{\theta}^{n+1},t^{n+1})-G_{\alpha}(\boldsymbol{\theta}^{*},t^{*})\} & \le(1-2\mu\min\{\beta,\gamma\})(G_{\alpha}(\boldsymbol{\theta}^{n},t^{n})-G_{\alpha}(\boldsymbol{\theta}^{*},t^{*}))\nonumber \\
 & \quad\quad+\dfrac{L}{2}(\max\{\beta,\gamma\})^{2}\dfrac{1+C_{T}^{2}}{\alpha^{2}},
\end{flalign}
where we have used that $\sup_{n\in\mathbb{N}_{T}}\mathbb{E}\{\Vert\nabla_{\boldsymbol{\theta}}\ell(f(\boldsymbol{x}^{n+1},\boldsymbol{\theta}^{n}),y^{n+1})\Vert_{2}^{2}\}\le C_{T}^{2}$.
Using that $\min\{\beta,\gamma\}<1$ and applying this inequality
recursively, we may easily see that
\begin{align}
\mathbb{E}\big\{ G_{\alpha}(\boldsymbol{\theta}^{T+1},t^{T+1})-G_{\alpha}(\boldsymbol{\theta}^{*},t^{*})\hspace{-0.5bp}\hspace{-1bp}\big\} & \le(1-2\mu\min\{\beta,\gamma\})^{T}(G_{\alpha}(\boldsymbol{\theta}^{0},t^{0})-G_{\alpha}(\boldsymbol{\theta}^{*},t^{*}))\nonumber \\
 & \quad\quad+\dfrac{(\max\{\beta,\gamma\})^{2}}{\min\{\beta,\gamma\}}\dfrac{L(1+C_{T}^{2})}{4\alpha^{2}\mu}.
\end{align}
The proof is complete.
\end{proof}
A couple of remarks regarding the assumptions and conclusions of Theorem
\ref{thm:CV@R_Linear} are essential at this point. First and foremost,
we should discuss the existence of an appropriate $\gamma$ satisfying
condition (\ref{eq:recursive}), which is of central importance in
the proof the theorem. Indeed, assume that there are choices of $\varepsilon$
and $\gamma$ such that, for every $n\in\mathbb{N}_{T}$,
\begin{equation}
\alpha\Big(1+\dfrac{\varepsilon}{\gamma}\Big)\le{\cal P}_{{\cal D}}({\cal A}(\boldsymbol{\theta}^{n},t^{n})),\label{eq:COND_1}
\end{equation}
which is a valid statement if and only if
\begin{equation}
\alpha\Big(1+\dfrac{\varepsilon}{\gamma}\Big)\le1\iff\dfrac{\alpha\varepsilon}{(1-\alpha)}\le\gamma,
\end{equation}
and equivalent to
\begin{equation}
1-\dfrac{1}{\alpha}{\cal P}_{{\cal D}}({\cal A}(\boldsymbol{\theta}^{n},t^{n}))\le-\dfrac{\varepsilon}{\gamma}.
\end{equation}
As a result (see proof of Theorem \ref{thm:CV@R_Linear}), by construction
of $\mathrm{CV@R}$-$\mathrm{SGD}$ we obtain
\begin{equation}
\mathbb{E}_{n}\{t^{n+1}|\mathscr{D}_{n}\}-t^{n}\equiv-\gamma\Big[1-\dfrac{1}{\alpha}{\cal P}_{{\cal D}}({\cal A}(\boldsymbol{\theta}^{n},t^{n}))\Big]\ge\varepsilon.
\end{equation}
Consequently, to satisfy (\ref{eq:recursive}), we may \textit{additionally}
demand that
\begin{equation}
\varepsilon>2\gamma\mu(t^{*}-t^{n})_{+},
\end{equation}
and noting that $t^{n}$ can be conservatively taken no less than
$l-(2\mu)^{-1}$, where $l$ denotes the lowest value of the loss
under consideration (this may be shown again by construction of $\mathrm{CV@R}$-$\mathrm{SGD}$),
we end up with the uniform upper limit
\begin{equation}
\gamma<\dfrac{\varepsilon}{2\mu(t^{*}-l)+1}.
\end{equation}
Overall, \textit{together with} (\ref{eq:COND_1}) we have the conditions
\begin{equation}
\dfrac{\alpha\varepsilon}{1-\alpha}\le\gamma<\dfrac{\varepsilon}{2\mu(t^{*}-l)+1},\label{eq:COND_2}
\end{equation}
from where it follows that it must also be the case that
\begin{equation}
\dfrac{\alpha}{1-\alpha}<\dfrac{1}{2\mu(t^{*}-l)+1}\iff t_{\alpha}^{*}-l<\dfrac{1-2\alpha}{2\alpha\mu}
\end{equation}
in order for such conditions on $\gamma$ to be meaningful. Lastly,
note that conditions (\ref{eq:COND_1}) and (\ref{eq:COND_2}) can
indeed be satisfied for particular choices of $\varepsilon$ and $\gamma$
when $\alpha$ is small enough.

Although these dependencies could seem fairly restrictive, they are
very reasonable, since in order for $\mathrm{CV@R}$-SGD to converge
fast, the condition $\ell(f(\boldsymbol{x}^{n+1},\boldsymbol{\theta}^{n}),y^{n+1})-t^{n}\ge0$
needs to be satisfied sufficiently often. But all this is reasonable
from a practical perspective as well: If $\alpha$ is closer to $1$
(risk-neutral setting), risky events are effectively smoothened, whereas,
if $\alpha$ approaches zero, only rare events matter, and an essentially
robust solution is sought, which does not really exhibit the dynamic
character of a risk-aware solution. Therefore, depending on the problem,
$\alpha$ should be chosen modestly, providing \textit{both} non-trivial
results \textit{and} fast linear convergence; from a conceptual point
of view, there is a certain logical \textit{balance to be respected
between moderatism and conservatism}.

Second, the set-restricted P\L{} inequality involved in Theorem \ref{thm:CV@R_Linear}
may still look mysterious, but is indeed useful. In fact, by Proposition
\ref{prop:SC_sat_scPL}, a byproduct of Theorem \ref{thm:CV@R_Linear}
is \textit{that $\mathrm{CV@R}$}-SGD\textit{ converges linearly to
fixed, user-tunable accuracy whenever $\ell(f(\boldsymbol{x},\cdot),y)$
is strongly convex and smooth for every $(\boldsymbol{x},y)$}, even
though $G_{\alpha}$ might not be strongly convex. This is especially
important, because it shows that classical problems, such as linear
least squares regression, can \textit{provably} be solved most efficiently
using SGD under risk-aware performance criteria, i.e., the $\mathrm{CV@R}$,
just as their risk-neutral counterparts (for instance, via the celebrated
Least-Mean-Squares (LMS) algorithm for linear least squares problems).

\section{Enforcing Smoothness}

There are two potential issues associated with the $\mathrm{CV@R}$
problem (\ref{eq:CVAR_min}) and the assumptions ensuring linear convergence
of\textit{ $\mathrm{CV@R}$}-SGD, as suggested in Theorem \ref{thm:CV@R_Linear}.
The \textit{first} is that there are useful cases where the demand
that ${\cal P}_{{\cal D}}(\ell(f(\boldsymbol{x},\bullet),y)=(\cdot))\equiv0$
on $\mathbb{R}^{m}\times\mathbb{R}$ (see Assumption \ref{Assumption1}.\ref{Ass12})
might not be satisfied; this happens, e.g., in classification problems
where the hypothesis class ${\cal F}$ contains \textit{hard} \textit{classifiers},
i.e., functions with binary or discrete range. The \textit{second}
issue is that the smoothness assumption on $G_{\alpha}$, essential
to obtain the rate promised by Theorem \ref{thm:CV@R_Linear}, might
not be easy to verify or even hold by merely assuming that the loss
$\ell(f(\boldsymbol{x},\cdot),y)$ is smooth; this is due to the presence
of the indicator ${\bf 1}_{{\cal A}(\bullet,\cdot)}(\boldsymbol{x},y)$
next to $\nabla_{\boldsymbol{\theta}}\ell(f(\boldsymbol{x},\bullet),y)$
in (\ref{eq:GRAD}). It turns out that these two issues are related,
and both may be mitigated by a rather simple strategy, which we now
discuss.

Consider an \textit{augmented example} $(\boldsymbol{x},y,w)$, where
$w\sim{\cal N}(0,\sigma^{2})$, $\sigma^{2}\hspace{-0.5bp}\hspace{-0.5bp}>\hspace{-0.5bp}\hspace{-0.5bp}0$,
is a \textit{fictitious target}, independent of $(\boldsymbol{x},y)$,
which we choose to use \textit{adversarially} during the training
process. In particular, we do that by defining the \textit{surrogate
loss} $\widetilde{\ell}:\mathbb{R}\times\mathbb{R}\times\mathbb{R}\rightarrow\mathbb{R}$
as
\begin{equation}
\widetilde{\ell}(f(\boldsymbol{x},\boldsymbol{\theta}),y,w)\triangleq\ell(f(\boldsymbol{x},\boldsymbol{\theta}),y)-w,
\end{equation}
Although such a surrogate loss is meaningless in the risk-neutral
setting (since $\mathbb{E}\{w\}\equiv0$), it \textit{provides regularization}
in risk-aware and, in particular, $\mathrm{CV@R}$ statistical learning.
In fact, it can be easily shown that, by choosing $\widetilde{\ell}$
as the loss, Assumption \ref{Assumption1}.\ref{Ass12} is always
satisfied, and the resulting objective function in problem (\ref{eq:CVAR_min})
is $L'$-smooth whenever $\ell(f(\boldsymbol{x},\cdot),y)$ is $G$-Lipschitz
and $L$-smooth, with
\begin{equation}
L'\equiv\dfrac{L\sigma{\textstyle \sqrt{2\pi}}+G^{2}}{\alpha\sigma{\textstyle \sqrt{2\pi}}}.
\end{equation}

To see those facts, observe that because $w$ is independent of $(\boldsymbol{x},y)$,
we may write
\begin{flalign}
{\cal P}_{\widetilde{{\cal D}}}(\widetilde{\ell}(f(\boldsymbol{x},\boldsymbol{\theta}),y,w)=t) & \equiv{\cal P}_{\widetilde{{\cal D}}}(\ell(f(\boldsymbol{x},\boldsymbol{\theta}),y)-w=t)\nonumber \\
 & \equiv\mathbb{E}_{{\cal P}_{{\cal D}}}\{{\cal P}_{\mathbb{R}}(\ell(f(\boldsymbol{x},\boldsymbol{\theta}),y)-t=w|\boldsymbol{x},y)\}\equiv0,
\end{flalign}
since $w$ is a continuous random variable. This shows that Assumption
\ref{Assumption1}.\ref{Ass12} is satisfied. Further, recall the
expression for the gradient $\nabla G_{\alpha}$ which, for the loss
$\widetilde{\ell}$ considered here, becomes
\begin{equation}
\hspace{-1bp}\hspace{-1bp}\hspace{-1bp}\hspace{-1bp}\hspace{-1bp}\hspace{-1bp}\hspace{-1bp}\hspace{-0.5bp}\nabla G_{\alpha}(\boldsymbol{\theta},t)\hspace{-1bp}\hspace{-0.5bp}=\hspace{-1bp}\hspace{-1bp}\hspace{-1bp}\hspace{-1bp}\begin{bmatrix}\dfrac{1}{\alpha}\mathbb{E}_{{\cal P}_{\widetilde{{\cal D}}}}\{{\bf 1}_{{\cal A}(\boldsymbol{\theta},t)}(\boldsymbol{x},y,w)\nabla_{\boldsymbol{\theta}}\widetilde{\ell}(f(\boldsymbol{x},\boldsymbol{\theta}),y,w)\}\\
-\dfrac{1}{\alpha}\mathbb{E}_{{\cal P}_{\widetilde{{\cal D}}}}\{{\bf 1}_{{\cal A}(\boldsymbol{\theta},t)}(\boldsymbol{x},y,w)\}+1
\end{bmatrix}\hspace{-1bp}\hspace{-1bp}\hspace{-0.5bp},\hspace{-1bp}\hspace{-1bp}\hspace{-1bp}\hspace{-1bp}\hspace{-1bp}
\end{equation}
where we additionally identify $\widetilde{{\cal D}}\triangleq\mathbb{R}^{d}\times\mathbb{R}\times\mathbb{R}$.
We first readily see that
\begin{flalign}
\mathbb{E}_{{\cal P}_{\widetilde{{\cal D}}}}\{{\bf 1}_{{\cal A}(\boldsymbol{\theta},t)}(\boldsymbol{x},y,w)\} & \equiv\mathbb{E}_{{\cal P}_{{\cal D}}}\{\mathbb{E}_{{\cal P}_{w}}\{{\bf 1}_{{\cal A}(\boldsymbol{\theta},t)}(\boldsymbol{x},y,w)|\boldsymbol{x},y\}\}\nonumber \\
 & \equiv\mathbb{E}_{{\cal P}_{{\cal D}}}\{{\cal P}_{w}(\ell(f(\boldsymbol{x},\boldsymbol{\theta}),y)-t>w|\boldsymbol{x},y)\}\nonumber \\
 & =\mathbb{E}_{{\cal P}_{{\cal D}}}\Big\{\Phi\Big(\dfrac{\ell(f(\boldsymbol{x},\boldsymbol{\theta}),y)-t}{\sigma}\Big)\hspace{-1bp}\hspace{-0.5bp}\Big\},
\end{flalign}
where $\Phi:\mathbb{R}\rightarrow[0,1]$ denotes the standard Gaussian
cumulative distribution function. In similar fashion, we also obtain
\begin{align}
\mathbb{E}_{{\cal P}_{\widetilde{{\cal D}}}}\{{\bf 1}_{{\cal A}(\boldsymbol{\theta},t)}(\boldsymbol{x},y,w)\nabla_{\boldsymbol{\theta}}\widetilde{\ell}(f(\boldsymbol{x},\boldsymbol{\theta}),y,w)\} & \equiv\mathbb{E}_{{\cal P}_{\widetilde{{\cal D}}}}\{{\bf 1}_{{\cal A}(\boldsymbol{\theta},t)}(\boldsymbol{x},y,w)\nabla_{\boldsymbol{\theta}}\ell(f(\boldsymbol{x},\boldsymbol{\theta}),y)\}\nonumber \\
 & \equiv\mathbb{E}_{{\cal P}_{{\cal D}}}\{\mathbb{E}_{{\cal P}_{w}}\{{\bf 1}_{{\cal A}(\boldsymbol{\theta},t)}(\boldsymbol{x},y,w)|\boldsymbol{x},y\}\nabla_{\boldsymbol{\theta}}\ell(f(\boldsymbol{x},\boldsymbol{\theta}),y)\}\nonumber \\
 & \equiv\mathbb{E}_{{\cal P}_{{\cal D}}}\{{\cal P}_{w}(\ell(f(\boldsymbol{x},\boldsymbol{\theta}),y)-t>w|\boldsymbol{x},y)\nabla_{\boldsymbol{\theta}}\ell(f(\boldsymbol{x},\boldsymbol{\theta}),y)\}\nonumber \\
 & \equiv\mathbb{E}_{{\cal P}_{{\cal D}}}\Big\{\Phi\Big(\dfrac{\ell(f(\boldsymbol{x},\boldsymbol{\theta}),y)-t}{\sigma}\Big)\nabla_{\boldsymbol{\theta}}\ell(f(\boldsymbol{x},\boldsymbol{\theta}),y)\Big\}.
\end{align}
Therefore, the gradient $\nabla G_{\alpha}$ may be equivalently represented
as
\begin{equation}
\hspace{-1bp}\hspace{-1bp}\hspace{-1bp}\hspace{-1bp}\hspace{-1bp}\hspace{-1bp}\hspace{-1bp}\hspace{-0.5bp}\nabla G_{\alpha}(\boldsymbol{\theta},t)\hspace{-1bp}\hspace{-0.5bp}=\hspace{-1bp}\hspace{-1bp}\hspace{-1bp}\hspace{-1bp}\begin{bmatrix}\dfrac{1}{\alpha}\mathbb{E}_{{\cal P}_{{\cal D}}}\Big\{\Phi\Big(\dfrac{\ell(f(\boldsymbol{x},\boldsymbol{\theta}),y)-t}{\sigma}\Big)\nabla_{\boldsymbol{\theta}}\ell(f(\boldsymbol{x},\boldsymbol{\theta}),y)\Big\}\\
-\dfrac{1}{\alpha}\mathbb{E}_{{\cal P}_{{\cal D}}}\Big\{\Phi\Big(\dfrac{\ell(f(\boldsymbol{x},\boldsymbol{\theta}),y)-t}{\sigma}\Big)\hspace{-1bp}\hspace{-0.5bp}\Big\}+1
\end{bmatrix}\hspace{-1bp}\hspace{-1bp}\hspace{-0.5bp}.\hspace{-1bp}\hspace{-1bp}\hspace{-1bp}\hspace{-1bp}\hspace{-1bp}
\end{equation}
Our claims above readily follow by exploiting this gradient representation.

Further, because it is true that \citep{Kalogerias2018b}
\begin{equation}
\mathbb{E}_{{\cal P}_{w}}\{(z-w)_{+}\}=\sigma\Big(\dfrac{z}{\sigma}\Phi\Big(\dfrac{z}{\sigma}\Big)+\phi\Big(\dfrac{z}{\sigma}\Big)\Big)\triangleq{\cal R}_{\sigma}(z),\quad\forall z\in\mathbb{R},
\end{equation}
where $\phi:\mathbb{R}\rightarrow\mathbb{R}_{+}$ denotes the standard
Gaussian density, and due to the fact that
\begin{equation}
(z)_{+}\le{\cal R}_{\sigma}(z)\le{\cal R}_{\sigma}(0)+(z)_{+}\equiv\dfrac{\sigma}{\sqrt{2\pi}}+(z)_{+},\quad\forall z\in\mathbb{R},
\end{equation}
we may readily derive \textit{uniform} estimates in $(\boldsymbol{\theta},t)$
\begin{equation}
\mathrm{CV@R}_{{\cal P}_{{\cal D}}}^{\alpha}[\ell(f(\boldsymbol{x},\boldsymbol{\theta}),y)]\le\mathrm{CV@R}_{{\cal P}_{\widetilde{{\cal D}}}}^{\alpha}[\widetilde{\ell}(f(\boldsymbol{x},\boldsymbol{\theta}),y,w)]\le\mathrm{CV@R}_{{\cal P}_{{\cal D}}}^{\alpha}[\ell(f(\boldsymbol{x},\boldsymbol{\theta}),y)]+\dfrac{\sigma}{\alpha\sqrt{2\pi}}.
\end{equation}
Then, similarly to Theorem \ref{thm:CV@R_Linear}, we obtain linear
convergence up to fixed accuracy 
\begin{equation}
\dfrac{(\max\{\beta,\gamma\})^{2}}{\min\{\beta,\gamma\}}\dfrac{(1+C_{T}^{2})}{4\alpha^{2}\mu}\dfrac{L\sigma{\textstyle \sqrt{2\pi}}+G^{2}}{\alpha\sigma{\textstyle \sqrt{2\pi}}}+\dfrac{\sigma}{\alpha{\textstyle \sqrt{2\pi}}}
\end{equation}
which by proper choice of $\sigma$ results in a quantity of the order
of
\begin{equation}
\Big(\sqrt{(\max\{\beta,\gamma\})^{2}/\min\{\beta,\gamma\}}\Big)\Big/\alpha^{2}.
\end{equation}
We observe that this result is slightly worse than that of Theorem
\ref{thm:CV@R_Linear}.

\section{\label{sec:Numerics}A Simple Numerical Example}

\begin{figure}
\centering\includegraphics[scale=0.6]{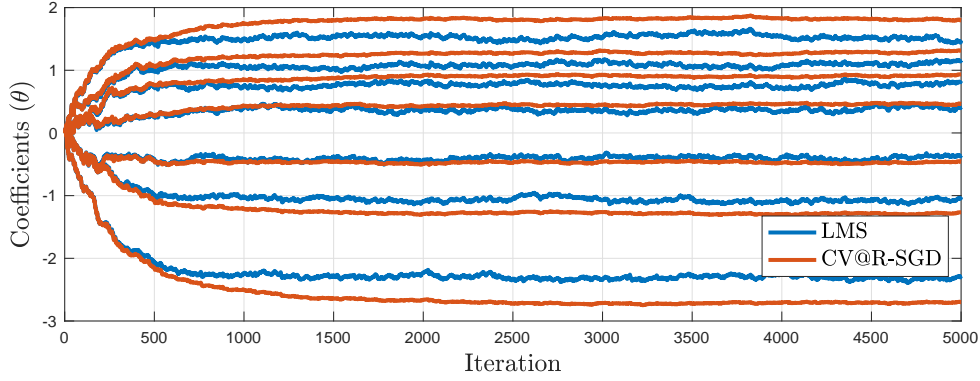}

\vspace{-8bp}
\caption{\label{fig:Comparison}Comparison between risk-neutral (LMS) and risk-aware
(\textit{$\mathrm{CV@R}$}-SGD) ridge regression: Evolution of iterates
$\{\boldsymbol{\theta}^{n}\}_{n}$.}
\end{figure}
In this section, we numerically demonstrate the behavior of \textit{$\mathrm{CV@R}$}-SGD,
confirming the validity of Theorem \ref{thm:CV@R_Linear}. To this
end, we consider the $\lambda$-strongly convex, risk-aware ridge
regression problem
\begin{equation}
\inf_{\boldsymbol{\theta}\in\mathbb{R}^{m}}\mathrm{CV@R}_{{\cal P}_{{\cal D}}}^{\alpha}\big[(y-\langle\boldsymbol{\theta},\boldsymbol{x}\rangle)^{2}+\lambda\Vert\boldsymbol{\theta}\Vert_{2}^{2}\big],\label{eq:Example}
\end{equation}
where $y\equiv\langle\boldsymbol{\theta}_{o},\boldsymbol{x}\rangle\in\mathbb{R}$
for a constant $\boldsymbol{\theta}_{o}\in\mathbb{R}^{7}$ and with
the elements of $\boldsymbol{x}\in\mathbb{R}^{7}$ being independent
uniform in $[0,2]$, $\lambda\equiv0.1$ and $\alpha\equiv0.2$. Therefore,
our goal is to find a $\boldsymbol{\theta}^{*}$\textit{ }which minimizes
the mean of the worst $80\%$ of all possible values of the random
error $(y-\langle\cdot,\boldsymbol{x}\rangle)^{2}+\lambda\Vert\cdot\Vert_{2}^{2}$.
Note that, for $\alpha\equiv1$, problem (\ref{eq:Example}) reduces
to ordinary ridge regression, and may be solved via the LMS algorithm.

Figs. \ref{fig:Comparison} and \ref{fig:Comparison-1} show the iterate
evolution as well as the behavior of the optimal prediction (test)
error for both \textit{$\mathrm{CV@R}$}-SGD (with stepsizes $\beta\equiv\alpha\hspace{-1bp}\hspace{-1bp}\times\hspace{-1bp}\hspace{-1bp}0.01$
and $\gamma\equiv0.001$) and the LMS scheme (with stepsize $\beta\equiv0.01$),
respectively. We observe that both algorithms converge at an essentially
identical \textit{noisy linear} \textit{rate}, in line with Theorem
\ref{thm:CV@R_Linear}. However, the solutions are radically different.
In fact, the risk-aware solution discovered by \textit{$\mathrm{CV@R}$}-SGD
\textit{dramatically} reduces the volatility of prediction error,
and provides prediction stability. Although this apparently comes
at the cost sacrificing mean performance, such sacrifice is fully
user-customizable by varying the \textit{$\mathrm{CV@R}$} level $\alpha$.
\begin{figure}
\centering\includegraphics[scale=0.65]{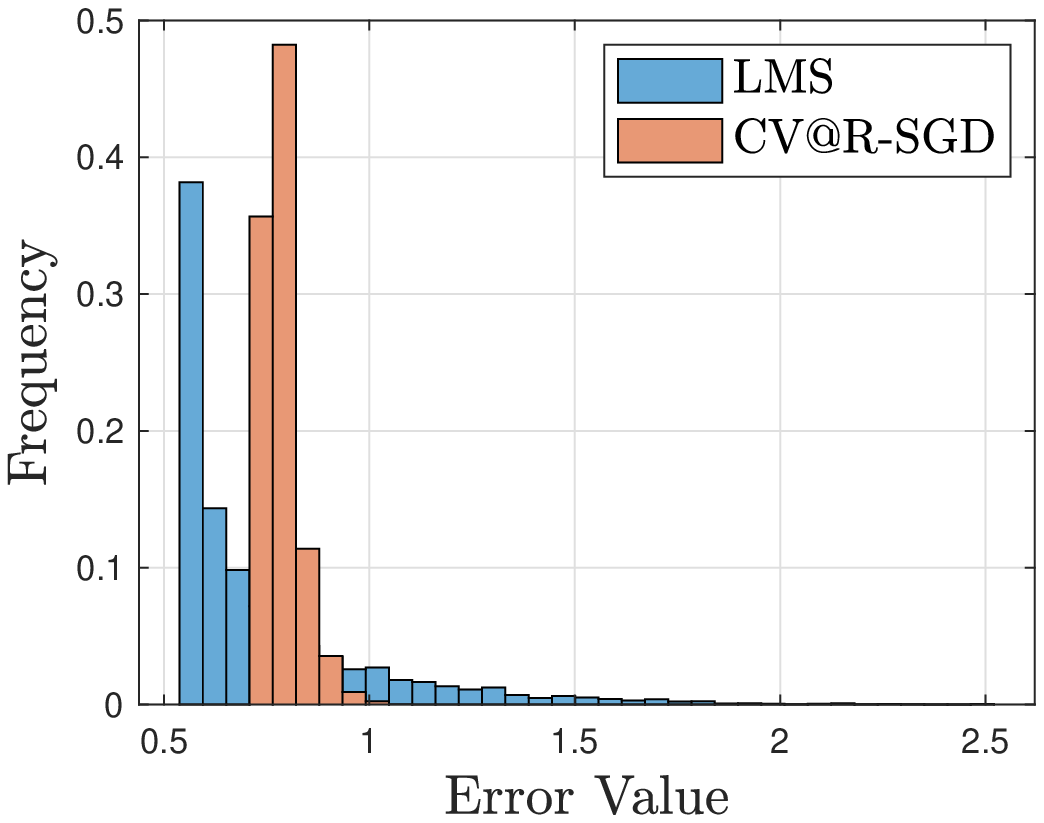}\hspace{30bp}\includegraphics[scale=0.65]{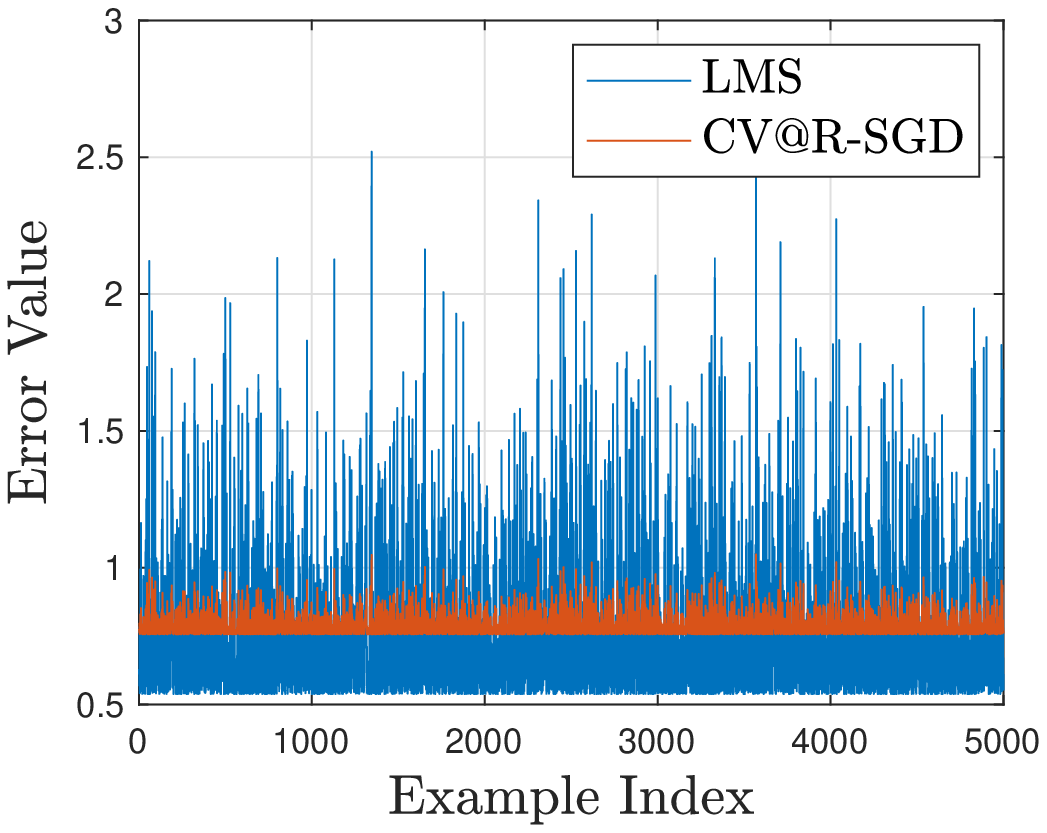}

\vspace{-8bp}
\caption{\label{fig:Comparison-1}Comparison between risk-neutral (LMS) and
risk-aware (\textit{$\mathrm{CV@R}$}-SGD) ridge regression: Histogram
(left) and actual values (right) of the test error.}
\end{figure}

\section{Conclusion}

In this work, we established noisy linear convergence of SGD for sequential\textbf{\textit{
$\mathrm{CV@R}$}} learning, for a large class of possibly nonconvex
loss functions satisfying a set-restricted P\L{} inequality, also
including all smooth and strongly convex losses as special cases.
This result disproves the belief that \textbf{\textit{$\mathrm{CV@R}$}}
learning is fundamentally difficult, and shows that classical learning
problems can be solved efficiently under \textbf{\textit{$\mathrm{CV@R}$
}}criteria, just as their risk-neutral versions. Our theory was also
illustrated via an indicative numerical example. Future work includes
the consideration of special learning settings such as linear least
squares, as well as other risk measures beyond \textbf{\textit{$\mathrm{CV@R}$.}}

\bibliographystyle{plainnat}
\phantomsection\addcontentsline{toc}{section}{\refname}\bibliography{library_fixed}

\end{document}